\documentclass[11pt, a4paper, logo, twocolumn, copyright]{googledeepmind}

\usepackage[authoryear, sort&compress, round]{natbib}
\usepackage[capitalise]{cleveref}
\usepackage{amsfonts}
\usepackage{titlesec}
\usepackage{amsmath}
\usepackage{graphicx}
\usepackage{paralist}
\usepackage{listings}
\usepackage{enumitem}
\usepackage{authblk}
\usepackage{ifthen}
\usepackage{defs}
\usepackage{float}
\usepackage{subcaption}
\usepackage{booktabs}
\usepackage[T1]{fontenc}
\newboolean{shownotes}
\setboolean{shownotes}{false}
\newcommand{\rohin}[1]{\ifthenelse{\boolean{shownotes}}{{\color{blue} RS: #1}}{}
}

\newcommand{\todo}[2][]{\textcolor{red}{\textbf{TODO\ifthenelse{ \equal{#1}{} }{}{(#1)}:} #2}}
\newcommand{\TODO}[2][]{\textcolor{red}{\textbf{TODO\ifthenelse{ \equal{#1}{} }{}{(#1)}:} #2}}
\newcommand{\norm}[1]{\left\lVert #1 \right\rVert}
\newtheorem{theorem}{Theorem}
\newtheorem{lemma}{Lemma}

\title{Jumping Ahead: Improving Reconstruction Fidelity with JumpReLU Sparse Autoencoders}

\correspondingauthor{srajamanoharan@google.com and neelnanda@google.com}

\author[*]{Senthooran Rajamanoharan}
\author[ \hspace{-0.7ex}]{Tom Lieberum$^\dagger$}
\author[ \hspace{-0.7ex}]{Nicolas Sonnerat}
\author[ \hspace{-0.7ex}]{Arthur Conmy}
\author[ \hspace{-0.7ex}]{Vikrant Varma}
\author[ \hspace{-0.7ex}]{J\'anos Kram\'ar}
\author[ \hspace{-0.7ex}]{Neel Nanda}

\affil[*]{: Core contributor. $^\dagger$: Core infrastructure contributor.}

\begin{abstract}
Sparse autoencoders (SAEs) are a promising unsupervised approach for identifying causally relevant and interpretable linear features in a language model's (LM) activations.
To be useful for downstream tasks, SAEs need to decompose LM activations faithfully; yet to be interpretable the decomposition must be sparse -- two objectives that are in tension.
In this paper, we introduce JumpReLU SAEs, which achieve state-of-the-art reconstruction fidelity at a given sparsity level on Gemma 2 9B activations, compared to other recent advances such as Gated and TopK SAEs.
We also show that this improvement does not come at the cost of interpretability through manual and automated interpretability studies.
JumpReLU SAEs are a simple modification of vanilla (ReLU) SAEs -- where we replace the ReLU with a discontinuous JumpReLU activation function -- and are similarly efficient to train and run.
By utilising straight-through-estimators (STEs) in a principled manner, we show how it is possible to train JumpReLU SAEs effectively despite the discontinuous JumpReLU function introduced in the SAE's forward pass. Similarly, we use STEs to directly train L0 to be sparse, instead of training on proxies such as L1, avoiding problems like shrinkage.

\end{abstract}
\begin{document}
\maketitle
\section{Introduction}
\label{sec:intro}

Sparse autoencoders (SAEs) allow us to find causally relevant and seemingly interpretable directions in the activation space of a language model \citep{bricken2023monosemanticity,cunningham2023sparse,templeton2024scaling}.
There is interest within the field of mechanistic interpretability in using sparse decompositions produced by SAEs for tasks such as circuit analysis \citep{marks2024sparse} and model steering \citep{conmy2024activation}.

SAEs work by finding approximate, sparse, linear decompositions of language model (LM) activations in terms of a large dictionary of basic ``feature'' directions. Two key objectives for a good decomposition \citep{bricken2023monosemanticity} are that it is sparse -- i.e.~that only a few elements of the dictionary are needed to reconstruct any given activation -- and that it is faithful -- i.e.~the approximation error between the original activation and recombining its SAE decomposition is ``small'' in some suitable sense. These two objectives are naturally in tension: for any given SAE training method and fixed dictionary size, it is typically not possible to increase sparsity without losing reconstruction fidelity.

One strand of recent research in training SAEs on LM activations \citep{rajamanoharan2024improving, gao2024scalingevaluatingsparseautoencoders, taggart} has been on finding improved SAE architectures and training methods that push out the Pareto frontier balancing these two objectives, while preserving other less quantifiable measures of SAE quality such as the interpretability or functional relevance of dictionary directions. A common thread connecting these recent improvements is the introduction of a thresholding or gating operation to determine which SAE features to use in the decomposition.

\begin{figure*}[t]
\centering
\includegraphics[width=\textwidth]{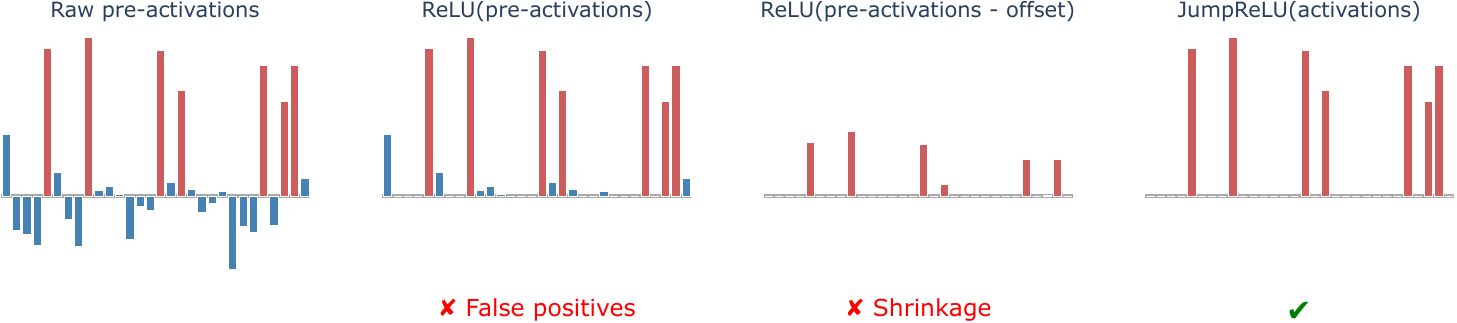}
\caption{A toy model illustrating why JumpReLU (or similar activation functions, such as TopK) are an improvement over ReLU for training sparse yet faithful SAEs. Consider a direction in which the encoder pre-activation is high when the corresponding feature is active and low, but not always negative, when the feature is inactive (far-left). Applying a ReLU activation function fails to remove all false positives (centre-left), harming sparsity. It is possible to get rid of false positives while maintaining the ReLU, e.g.~by decreasing the encoder bias (centre-right), but this leads to feature magnitudes being systematically underestimated, harming fidelity. The JumpReLU activation function (far-right) provides an independent threshold below which pre-activations are screened out, minimising false positives, while leaving pre-activations above the threshold unaffected, improving fidelity.}
\label{fig:why-jumprelu}
\end{figure*}

In this paper, we introduce \textbf{JumpReLU SAEs} -- a small modification of the original, ReLU-based SAE architecture \citep{ng2011sparse} where the SAE encoder's ReLU activation function is replaced by a JumpReLU activation function \citep{erichson2019jumprelu}, which zeroes out pre-activations below a positive threshold (see \cref{fig:why-jumprelu}). Moreover, we train JumpReLU SAEs using a loss function that is simply the weighted sum of a L2 reconstruction error term and a L0 sparsity penalty, eschewing easier-to-train proxies to L0, such as L1, and avoiding the need for auxiliary tasks to train the threshold.

Our key insight is to notice that although such a loss function is piecewise-constant with respect to the threshold -- and therefore provides zero gradient to train this parameter -- the derivative of the \emph{expected loss} can be analytically derived, and is generally non-zero, albeit it is expressed in terms of probability densities of the feature activation distribution that need to be estimated. We show how to use straight-through-estimators (STEs; \citet{bengio2013estimatingpropagatinggradientsstochastic}) to estimate the gradient of the expected loss in an efficient manner, thus allowing JumpReLU SAEs to be trained using standard gradient-based methods.

\begin{figure*}[t]
\centering
\includegraphics[width=\textwidth]{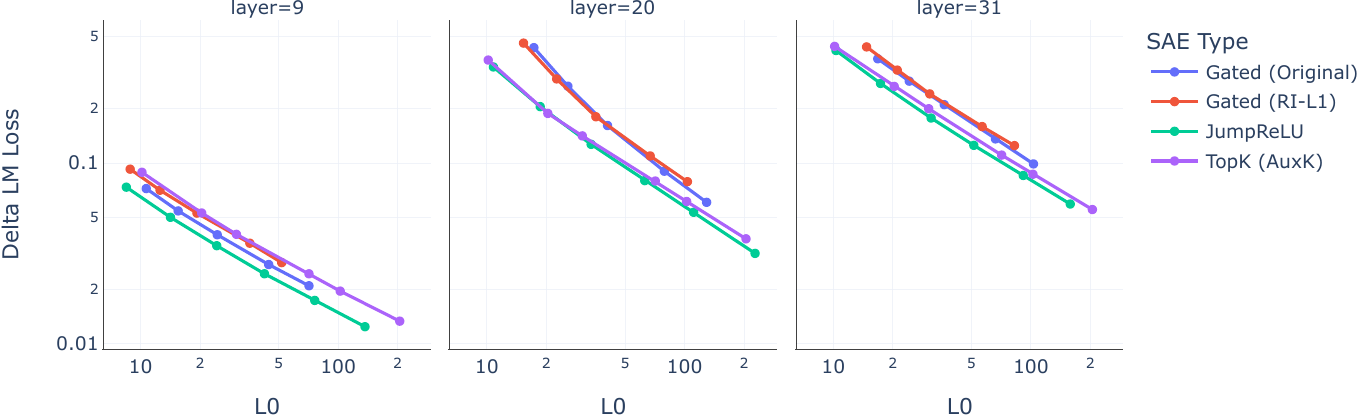}
\caption{JumpReLU SAEs offer reconstruction fidelity that equals or exceeds Gated and TopK SAEs at a fixed level of sparsity. These results are for SAEs trained on the residual stream after layers 9, 20 and 31 of Gemma 2 9B. See \cref{fig:paretos-mlp} and \cref{fig:paretos-attn} for analogous plots for SAEs trained on MLP and attention output activations at these layers.}
\label{fig:paretos}
\end{figure*}

We evaluate JumpReLU, Gated and TopK \citep{gao2024scalingevaluatingsparseautoencoders} SAEs on Gemma 2 9B \citep{gemma2_2024} residual stream, MLP output and attention output activations at several layers (\cref{fig:paretos}). At any given level of sparsity, we find JumpReLU SAEs consistently provide more faithful reconstructions than Gated SAEs. JumpReLU SAEs also provide reconstructions that are at least as good as, and often slightly better than, TopK SAEs. Similar to simple ReLU SAEs, JumpReLU SAEs only require a single forward and backward pass during a training step and have an elementwise activation function (unlike TopK, which requires a partial sort), making them more efficient to train than either Gated or TopK SAEs.

Compared to Gated SAEs, we find both TopK and JumpReLU tend to have more features that activate very frequently -- i.e.~on more than 10\% of tokens (\cref{fig:high-density-features}). Consistent with prior work evaluating TopK SAEs \citep{anthropic_june_update} we find these high frequency JumpReLU features tend to be less interpretable, although interpretability does improve as SAE sparsity increases. Furthermore, only a small proportion of SAE features have very high frequencies: fewer than 0.06\% in a 131k-width SAE. We also present the results of manual and automated interpretability studies indicating that randomly chosen JumpReLU, TopK and Gated SAE features are similarly interpretable.
\section{Preliminaries}
\label{sec:preliminaries}


\paragraph{SAE architectures} SAEs sparsely decompose language model activations $\mathbf{x}\in\mathbb{R}^n$ as a linear combination of a \emph{dictionary} of $M\gg n$ \emph{learned feature} directions and then reconstruct the original activations using a pair of encoder and decoder functions $\left(\mathbf{f}, \hat{\mathbf{x}}\right)$ defined by:
\begin{align}
    \mathbf{f}(\mathbf{x}) &:= \sigma \left(\Wenc\mathbf{x} + \benc \right), \label{eq:encoder}\\
    \hat{\mathbf{x}}({\mathbf{f}}) &:= \Wdec {\mathbf{f}} + \bdec\label{eq:decoder}.
\end{align}
In these expressions, $\mathbf{f}(\mathbf{x}) \in \mathbb{R}^M$ is a sparse, non-negative vector of feature magnitudes present in the input activation $\mathbf{x}$, whereas $\hat{\mathbf{x}}(\mathbf{f}) \in \mathbb{R}^n$ is a reconstruction of the original activation from a feature representation $\mathbf{f} \in \mathbb{R}^M$. The columns of $\Wdec$, which we denote by $\mathbf{d}_i$ for $i=1\ldots M$, represent the dictionary of directions into which the SAE decomposes $\mathbf{x}$. We also use $\pibf(\mathbf{x})$ in this text to denote the encoder's pre-activations:
\begin{equation}
    \pibf(\textbf{x}) := \Wenc \textbf{x} + \benc.
\end{equation}

\paragraph{Activation functions} The activation function $\sigma$ varies between architectures: \citet{bricken2023monosemanticity} and \citet{templeton2024scaling} use the ReLU activation function, whereas TopK SAEs \citep{gao2024scalingevaluatingsparseautoencoders} use a TopK activation function (which zeroes out all but the top $K$ pre-activations). Gated SAEs \citep{rajamanoharan2024improving} in their general form do not fit the specification of \cref{eq:encoder}; however with weight sharing between the two encoder kernels, they can be shown \citep[Appendix E]{rajamanoharan2024improving} to be equivalent to using a JumpReLU activation function, defined as
\begin{equation}
\jr_\theta(z) := z\,H(z - \theta)
\label{eq:jump-relu}
\end{equation}
where $H$ is the Heaviside step function\footnote{$H(z)$ is one when $z>0$ and zero when $z<0$. Its value when $z=0$ is a matter of convention -- unimportant when $H$ appears within integrals or integral estimators, as is the case in this paper.} when  $\theta > 0$ is the JumpReLU's threshold, below which pre-activations are set to zero, as shown in \cref{fig:jumprelu-function}.

\begin{figure}[t]
\centering
\includegraphics[width=0.6\columnwidth]{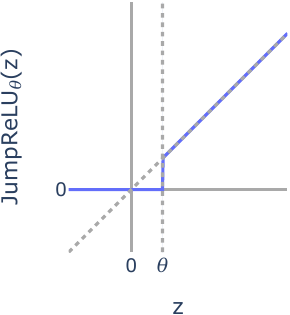}
\caption{The JumpReLU activation function zeroes inputs below the threshold, $\theta$, and is an identity function for inputs above the threshold.}
\label{fig:jumprelu-function}
\end{figure}

\paragraph{Loss functions} Language model SAEs are trained to reconstruct samples from a large dataset of language model activations $\mathbf{x} \sim \mathcal{D}$ typically using a loss function of the form
\begin{equation}
    \mathcal{L}(\mathbf{x}) := \underbrace{\norm{\mathbf{x}-\hat{\mathbf{x}}(\mathbf{f}(\mathbf{x}))}_2^2}_{\mathcal{L}_\text{reconstruct}} + \underbrace{\lambda\, S(\mathbf{f}(\mathbf{x}))}_{\mathcal{L}_\text{sparsity}} + \mathcal{L}_\text{aux},
    \label{eq:typical-loss}
\end{equation}
where $S$ is a function of the feature coefficients that penalises non-sparse decompositions and the \emph{sparsity coefficient} $\lambda$ sets the trade-off between sparsity and reconstruction fidelity. Optionally, auxiliary terms in the loss function, $\mathcal{L}_\text{aux}$ may be included for a variety of reasons, e.g.~to help train parameters that would otherwise not receive suitable gradients (used for Gated SAEs) or to resurrect unproductive (``dead'') feature directions (used for TopK). Note that TopK SAEs are trained without a sparsity penalty, since the TopK activation function directly enforces sparsity.

\paragraph{Sparsity penalties} Both the ReLU SAEs of \citet{bricken2023monosemanticity} and Gated SAEs use the L1-norm $S(\mathbf{f}) := \norm{\mathbf{f}}_1$ as a sparsity penalty. While this has the advantage of providing a useful gradient for training (unlike the L0-norm), it has the disadvantage of penalising feature magnitudes in addition to sparsity, which harms reconstruction fidelity \citep{rajamanoharan2024improving, wright2024addressing}.

The L1 penalty also fails to be invariant under reparameterizations of a SAE; by scaling down encoder parameters and scaling up decoder parameters accordingly, it is possible to arbitrarily shrink feature magnitudes, and thus the L1 penalty, without changing either the number of active features or the SAE's output reconstructions. As a result, it is necessary to impose a further constraint on SAE parameters during training to enforce sparsity: typically this is achieved by constraining columns of the decoder weight matrix $\mathbf{d}_i$ to have unit norm \citep{bricken2023monosemanticity}. \citet{conerly2024trainingsaes} introduce a modification of the L1 penalty, where feature coefficients are weighted by the norms of the corresponding dictionary directions, i.e.
\begin{equation}S_\rilo(\mathbf{f}) := \sum_{i=1}^{M} f_i \norm{\mathbf{d}_i}_2.
\label{eq:rilo-penalty}
\end{equation}
We call this the \emph{reparameterisation-invariant L1} (RI-L1) sparsity penalty, since this penalty is invariant to SAE reparameterisation, making it unnecessary to impose constraints on $\norm{\mathbf{d}_i}_2$.

\paragraph{Kernel density estimation} Kernel density estimation (KDE; \citet{parzen1962,wasserman2010statistics}) is a technique for empirically estimating probability densities from a finite sample of observations. Given $N$ samples $x_{1\ldots N}$ of a random variable $X$, one can form a kernel density estimate of the probability density $p_X(x)$ using an estimator of the form $\hat{p}_X(x) := \tfrac{1}{N\varepsilon}\sum_{\alpha=1}^N K\left(\tfrac{x - x_\alpha}{\varepsilon}\right)$, where $K$ is a non-negative function that satisfies the properties of a centred, positive-variance probability density function and $\varepsilon$ is the kernel \emph{bandwidth} parameter.\footnote{I.e.~$K(x) \geq 0$, $\int_{-\infty}^{\infty}K(x)\mathrm{d}x = 1$, $\int_{-\infty}^{\infty}x\,K(x)\mathrm{d}x = 0$ and $\int_{-\infty}^{\infty}x^2 K(x) \mathrm{d}x > 0$.} In this paper we will be actually be interested in estimating quantities like $v(y) = \mathbb{E}[f(X, Y) \vert Y=y] p_Y(y)$ for jointly distributed random variables $X$ and $Y$ and arbitrary (but well-behaved) functions $f$. Following a similar derivation as in \citet[Chapter 20]{wasserman2010statistics}, it is straightforward to generalise KDE to estimate $v(y)$ using the estimator
\begin{equation}
\hat{v}(y) := \frac{1}{N\varepsilon} \sum_{\alpha=1}^N f(x_\alpha, y_\alpha) K\left(\frac{y - y_\alpha}{\varepsilon}\right).
\label{eq:generalised-kde}
\end{equation}
\section{JumpReLU SAEs}
\label{sec:jumprelu-saes}

\begin{figure*}[t]
\centering
\begin{subfigure}[t]{0.5\textwidth}
    \raggedright
    \includegraphics[width=0.9\textwidth]{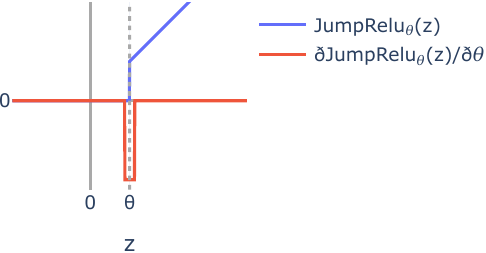}
\end{subfigure}%
~ 
\begin{subfigure}[t]{0.5\textwidth}
    \raggedleft
    \includegraphics[width=0.9\textwidth]{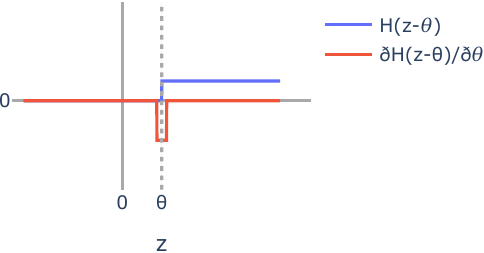}
\end{subfigure}
\caption{The JumpReLU activation function (left) and the Heaviside step function (right) used to calculate the sparsity penalty are piecewise constant with respect to the JumpReLU threshold. Therefore, in order to be able to train a JumpReLU SAE, we define the pseudo-derivatives illustrated in these plots and defined in \cref{eq:jr-ste} and \cref{eq:step-ste}, which approximate the Dirac delta functions present in the actual (weak) derivatives of the JumpReLU and Heaviside functions. These pseudo-derivatives provide a gradient signal to the threshold whenever pre-activations are within a small window of width $\varepsilon$ around the threshold. Note these plots show the profile of these pseudo-derivatives in the $z$, not $\theta$ direction, as $z$ is the stochastic input that is averaged over when computing the mean gradient.}
\label{fig:fig-ste}
\end{figure*}

A JumpReLU SAE is a SAE of the standard form \cref{eq:encoder} with a JumpReLU activation function:
\begin{equation}
\mathbf{f}(\mathbf{x}) := \jr_\thetabf\left(\Wenc \mathbf{x} + \benc \right).
\label{eq:jr-encoder}
\end{equation}
Compared to a ReLU SAE, it has an extra positive vector-valued parameter $\thetabf \in \mathbb{R}_+^M$ that specifies, for each feature $i$, the threshold that encoder pre-activations need to exceed in order for the feature to be deemed active.

Similar to the gating mechanism in Gated SAEs and the TopK activation function in TopK SAEs, the threshold $\thetabf$ gives JumpReLU SAEs the means to separate out deciding which features are active from estimating active features' magnitudes, as illustrated in \cref{fig:why-jumprelu}.

We train JumpReLU SAEs using the loss function
\begin{equation}
    \mathcal{L}(\mathbf{x}) := \underbrace{\norm{\mathbf{x}-\hat{\mathbf{x}}(\mathbf{f}(\mathbf{x}))}_2^2}_{\mathcal{L}_\text{reconstruct}} + \underbrace{\lambda\norm{\mathbf{f}(\mathbf{x})}_0}_{\mathcal{L}_\text{sparsity}}.
    \label{eq:jr-loss}
\end{equation}
This is a loss function of the standard form \cref{eq:typical-loss} where crucially we are using a L0 sparsity penalty to avoid the limitations of training with a L1 sparsity penalty \citep{wright2024addressing, rajamanoharan2024improving}. Note that we can also express the L0 sparsity penalty in terms of a Heaviside step function on the encoder's pre-activations $\pibf(\mathbf{x})$:
\begin{equation}
    \mathcal{L}_\text{sparsity} := \lambda\norm{\mathbf{f}(\mathbf{x})}_0 = \lambda \sum_{i=1}^{M}H(\pi_i(\mathbf{x}) - \theta_i).
    \label{eq:l0-step-function}
\end{equation}
The relevance of this will become apparent shortly.

The difficulty with training using this loss function is that it provides no gradient signal for training the threshold: $\thetabf$ appears only within the arguments of Heaviside step functions in both $\mathcal{L}_\text{reconstruct}$ and $\mathcal{L}_\text{sparsity}$.\footnote{The L0 sparsity penalty also provides no gradient signal for the remaining SAE parameters, but this is not necessarily a problem. It just means that the remaining SAE parameters are encouraged purely to reconstruct input activations faithfully, not worrying about sparsity, while sparsity is taken care of by the threshold parameter $\thetabf$. This is analogous to TopK SAEs, where similarly the main SAE parameters are trained solely to reconstruct faithfully, while sparsity is enforced by the TopK activation function.}
Our solution is to use straight-through-estimators (STEs; \citet{bengio2013estimatingpropagatinggradientsstochastic}), as illustrated in \cref{fig:fig-ste}. Specifically, we define the following pseudo-derivative for $\jr_\theta(z)$:\footnote{We use the notation $\pseudopartial / \pseudopartial z$ to denote pseudo-derivatives, to avoid conflating them with actual partial derivatives for these functions.}
\begin{equation}
    \frac{\pseudopartial}{\pseudopartial \theta}  \jr_\theta(z) := -\frac{\theta}{\varepsilon} K\left(\frac{z - \theta}{\varepsilon}\right)
\label{eq:jr-ste}
\end{equation}
and the following pseudo-derivative for the Heaviside step function appearing in the L0 penalty:
\begin{equation}
    \frac{\pseudopartial}{\pseudopartial \theta} H(z - \theta) :=  -\frac{1}{\varepsilon} K\left(\frac{z - \theta}{\varepsilon}\right).
\label{eq:step-ste}
\end{equation}
In these expressions, $K$ can be any valid kernel function (see \cref{sec:preliminaries}) -- i.e.~it needs to satisfy the properties of a centered, finite-variance probability density function. In our experiments, we use the rectangle function, $\text{rect}(z) := H\left(z+\tfrac{1}{2}\right) - H\left(z-\tfrac{1}{2}\right)$ as our kernel; however similar results can be obtained with other common kernels, such as the triangular, Gaussian or Epanechnikov kernel (see \cref{app:other-kernels}). As we show in \cref{sec:ste-kde-connection}, the hyperparameter $\varepsilon$ plays the role of a KDE bandwidth, and needs to be selected accordingly: too low and gradient estimates become too noisy, too high and estimates become too biased.\footnote{For the experiments in this paper, we swept this parameter and found $\varepsilon=0.001$ (assuming a dataset normalised such that $\mathbb{E}_\mathbf{x}[\mathbf{x}^2] = 1$) works well across different models, layers and sites. However, we suspect there are more principled ways to determine this parameter, borrowing from the literature on KDE bandwidth selection.}

Having defined these pseudo-derivatives, we train JumpReLU SAEs as we would any differentiable model, by computing the gradient of the loss function in \cref{eq:jr-loss} over batches of data (remembering to apply these pseudo-derivatives in the backward pass), and sending the batch-wise mean of these gradients to the optimiser in order to compute parameter updates.

In \cref{app:pseudo-code} we provide pseudocode for the JumpReLU SAE's forward pass, loss function and for implementing the straight-through-estimators defined in \cref{eq:jr-ste} and \cref{eq:step-ste} in an autograd framework like Jax \citep{jax2018github} or PyTorch \citep{paszke2019pytorchimperativestylehighperformance}.

\section{How STEs enable training through the jump}
\label{sec:ste-kde-connection}

Why does this work? The key is to notice that during SGD, we actually want to estimate the gradient of the \emph{expected} loss, $\mathbb{E}_\mathbf{x}\left[\mathcal{L}_\thetabf(\mathbf{x})\right]$, in order to calculate parameter updates;\footnote{In this section, we write the JumpReLU loss as $\mathcal{L}_\thetabf (\mathbf{x})$ to make explicit its dependence on the threshold parameter $\thetabf$.} Although the loss itself is piecewise constant with respect to the threshold parameters -- and therefore has zero gradient -- the expected loss is not.

As shown in \cref{app:expected-loss-derivative}, we can differentiate expected loss with respect to $\thetabf$ analytically to obtain
\begin{equation}
    \frac{\partial \mathbb{E}_\mathbf{x}\left[\mathcal{L}_\thetabf(\mathbf{x})\right]}{\partial \theta_i} 
    = \left(\mathbb{E}_\mathbf{x}\left[I_i(\mathbf{x}) \vert \pi_i(\mathbf{x}) = \theta_i \right] - \lambda\right) p_i(\theta_i),
    \label{eq:exact-derivative}
\end{equation}
where $p_i$ is the probability density function for the distribution of feature pre-activations $\pi_i(\mathbf{x})$ and
\begin{equation}
    I_i(\mathbf{x}) := 2 \theta_i \mathbf{d}_i \cdot (\mathbf{x} - \hat{\mathbf{x}}(\mathbf{f}(\mathbf{x}))),
    \label{eq:i-definition}
\end{equation}
recalling that $\mathbf{d}_i$ is the column of $\Wdec$ corresponding to feature $i$.\footnote{Intuitively, the first term in \cref{eq:exact-derivative} measures the rate at which the expected reconstruction loss would increase if we increase $\theta_i$ -- thereby pushing a small number of features that are currently used for reconstruction below the updated threshold. Similarly, the second term is $-\lambda$ multiplied by the rate at which the mean number of features used for reconstruction (i.e.~mean L0) would \emph{decrease} if we increase the threshold $\theta_i$. The density $p_i(\theta_i)$ comes into play because impact of a small change in $\theta_i$ on either the reconstruction loss or sparsity depends on how often feature activations occur very close to the current threshold.}

In order to train JumpReLU SAEs, we need to estimate the gradient as expressed in \cref{eq:exact-derivative} from batches of input activations, $\mathbf{x}_1, \mathbf{x}_2, \ldots, \mathbf{x}_N$. To do this, we can use a generalised KDE estimator of the form \cref{eq:generalised-kde}. This gives us the following estimator of the expected loss's gradient with respect to $\thetabf$:
\begin{equation}
    \frac{1}{N\varepsilon}\sum_{\alpha=1}^{N} \left\{I_i(\mathbf{x_\alpha}) - \lambda\right\} K \left(\frac{
    \pi_i(\mathbf{x}_\alpha)-\theta_i}{\varepsilon}\right).
    \label{eq:kde-estimate}
\end{equation}
As we show in \cref{app:ste-gives-kde}, when we instruct autograd to use the pseudo-derivatives defined in \cref{eq:jr-ste,eq:step-ste} in the backward pass, this is precisely the batch-wise mean gradient that gets calculated -- and used by the optimiser to update $\thetabf$ -- in the training loop.

In other words, training with straight-through-estimators as described in \cref{sec:jumprelu-saes} is equivalent to estimating the true gradient of the expected loss, as given in \cref{eq:exact-derivative}, using the kernel density estimator defined in \cref{eq:kde-estimate}.

\section{Evaluation}
\label{sec:experiments}

In this section, we compare JumpReLU SAEs to Gated and TopK SAEs across a range of evaluation metrics.\footnote{We did not include ProLU SAEs \citep{taggart} in our comparisons, despite their similarities to JumpReLU SAEs, because prior work has established that ProLU SAEs do not produce as faithful reconstructions as Gated or TopK SAEs at a given sparsity level \citep{gao2024scalingevaluatingsparseautoencoders}.}

To make these comparisons, we trained multiple 131k-width SAEs (with a range of sparsity levels) of each type (JumpReLU, Gated and TopK) on activations from Gemma 2 9B (base). Specifically, we trained SAEs on residual stream, attention output and MLP output sites after layers 9, 20 and 31 of the model (zero-indexed).

We trained Gated SAEs using two different loss functions. Firstly, we used the original Gated SAE loss in \citet{rajamanoharan2024improving}, which uses a L1 sparsity penalty and requires resampling \citep{bricken2023monosemanticity} -- periodic re-initialisation of dead features -- in order to train effectively. Secondly, we used a modified Gated SAE loss function that replaces the L1 sparsity penalty with the RI-L1 sparsity penalty described in \cref{sec:preliminaries}; see \cref{app:gated-sfn} for details. With this modified loss function, we no longer need to use resampling to avoid dead features.

We trained TopK SAEs using the AuxK auxiliary loss described in \citet{gao2024scalingevaluatingsparseautoencoders} with $K_\text{aux}=512$, which helps reduce the number of dead features. We also used an approximate algorithm for computing the top $K$ activations \citep{chern2022tpuknnknearestneighbor} -- implemented in JAX as \texttt{jax.lax.approx\_max\_k} -- after finding it produces similar results to exact TopK while being much faster (\cref{app:top-k}).

All SAEs used in these evaluations were trained over 8 billion tokens; by this point, they had all converged, as confirmed by inspecting their training curves. See \cref{app:further-training-details} for further details of our training methodology.

\subsection{Evaluating the sparsity-fidelity trade-off}

\paragraph{Methodology} For a fixed SAE architecture and dictionary size, we trained SAEs of varying levels of sparsity by sweeping either the sparsity coefficient $\lambda$ (for JumpReLU or Gated SAEs) or $K$ (for TopK SAEs). We then plot curves showing, for each SAE architecture, the level of reconstruction fidelity attainable at a given level of sparsity.

\paragraph{Metrics} We use the mean L0-norm of feature activations, $\mathbb{E}_{\mathbf{x}} \norm{\mathbf{f}(\mathbf{x})}_0$, as a measure of sparsity. To measure reconstruction fidelity, we use two metrics:
\begin{compactitem}
\item Our primary metric is delta LM loss, the increase in the cross-entropy loss experienced by the LM when we splice the SAE into the LM's forward pass.
\item As a secondary metric, we also present in \cref{fig:paretos-fvu} curves that use fraction of variance unexplained (FVU) -- also called the normalized loss \citep{gao2024scalingevaluatingsparseautoencoders} as a measure of reconstruction fidelity. This is the mean reconstruction loss $\mathcal{L}_\text{reconstruct}$ of a SAE normalised by the reconstruction loss obtained by always predicting the dataset mean.
\end{compactitem}
All metrics were computed on 2,048 sequences of length 1,024, after excluding special tokens (pad, start and end of sequence) when aggregating the results.

\paragraph{Results} \cref{fig:paretos} compares the sparsity-fidelity trade-off for JumpReLU, Gated and TopK SAEs trained on Gemma 2 9B residual stream activations. JumpReLU SAEs consistently offer similar or better fidelity at a given level of sparsity than TopK or Gated SAEs. Similar results are obtained for SAEs of each type trained on MLP or attention output activations, as shown in \cref{fig:paretos-mlp} and \cref{fig:paretos-attn} in \cref{app:more-benchmarking-results}.

\subsection{Feature activation frequencies}

For a given SAE, we are interested in both the proportion of learned features that are active very frequently and the proportion of features that are almost never active (``dead'' features). Prior work has found that TopK SAEs tend to have more high frequency features than Gated SAEs \citep{anthropic_june_update}, and that these features tend to be less interpretable when sparsity is also low.

\paragraph{Methodology} We collected SAE feature activation statistics over 10,000 sequences of length 1,024, and computed the frequency with which individual features fire on a randomly chosen token (excluding special tokens).

\begin{figure*}[t]
\centering
\begin{subfigure}[t]{\textwidth}
    \centering
    \includegraphics[width=\textwidth]{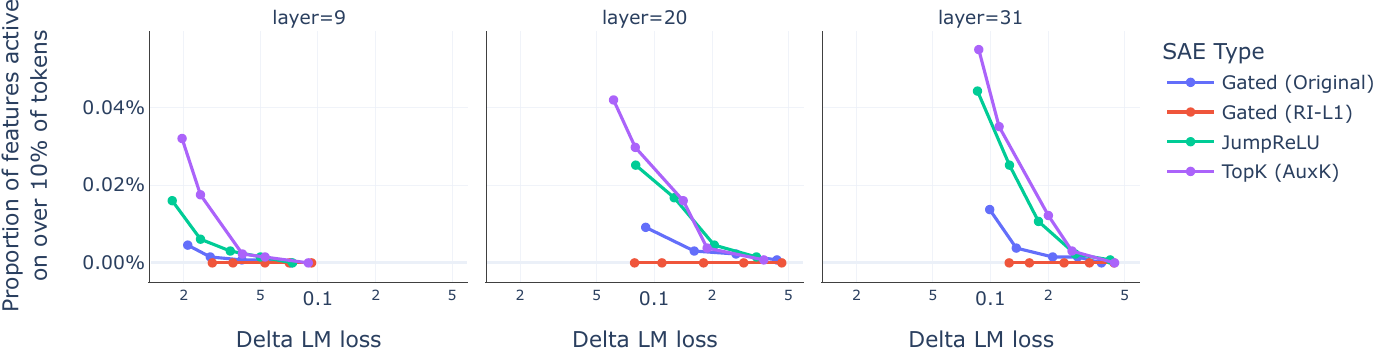}
\end{subfigure}%
\vspace{16pt}
\begin{subfigure}[t]{\textwidth}
    \centering
    \includegraphics[width=\textwidth]{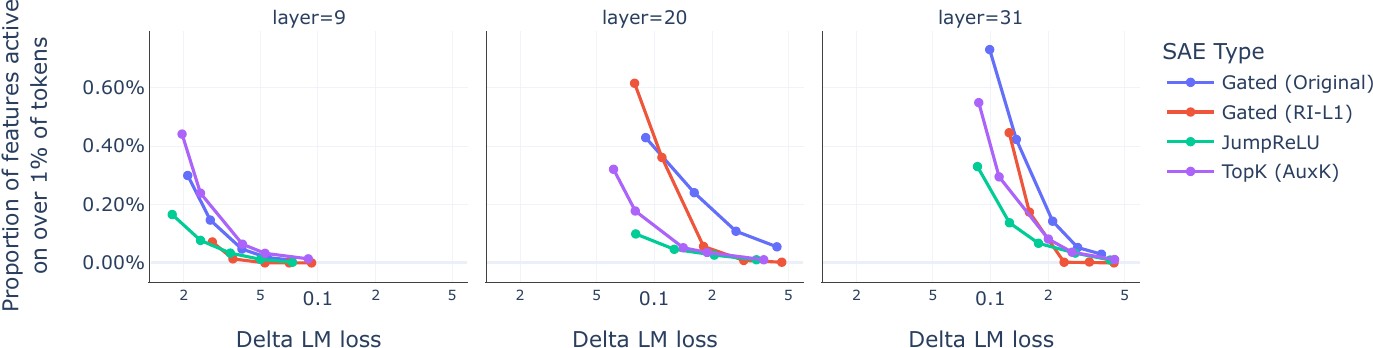}
\end{subfigure}
\caption{The proportion of features that activate very frequently versus delta LM loss by SAE type for Gemma 2 9B residual stream SAEs. TopK and JumpReLU SAEs tend to have relatively more very high frequency features -- those active on over 10\% of tokens (top) -- than Gated SAEs. If we instead count features that are active on over 1\% of tokens (bottom), the picture is more mixed: Gated SAEs can have more of these high (but not necessarily very high) features than JumpReLU SAEs, particularly in the low loss (and therefore lower sparsity) regime.}
\label{fig:high-density-features}
\end{figure*}

\paragraph{Results} \cref{fig:high-density-features} shows, for JumpReLU, Gated and TopK SAEs, how the fraction of high frequency features varies with SAE fidelity (as measured by delta LM loss). TopK and JumpReLU SAEs consistently have more very high frequency features -- features that activate on over 10\% of tokens (top plot) -- than Gated SAEs, although the fraction drops close to zero for SAEs in the low fidelity / high sparsity regime. On the other hand, looking at features that activate on over 1\% of tokens (a wider criterion), Gated SAEs have comparable numbers of such features to JumpReLU SAEs (bottom plot), with considerably more in the low delta LM loss / higher L0 regime (although all these SAEs have L0 less than 100, i.e.~are reasonably sparse). Across all layers and frequency thresholds, JumpReLU SAEs have either similar or fewer high frequency features than TopK SAEs. Finally, it is worth noting that in all cases the number of high frequency features remains low in proportion to the widths of these SAEs, with fewer than 0.06\% of features activating more than 10\% of the time even for the highest L0 SAEs.

\cref{fig:dead-features} compares the proportion of ``dead'' features -- which we defined to be features that activate on fewer than one in $10^7$ tokens -- between JumpReLU, Gated and TopK SAEs. Both JumpReLU SAEs and TopK SAEs (trained with the AuxK loss) consistently have few dead features, without the need for resampling.

\subsection{Interpretability of SAE features}
\label{sec:interp-of-features}
Exactly how to assess the quality of the features learned by an SAE is an open research question. Existing work has focused on the activation patterns of features with particular emphasis paid to sequences a feature activates most strongly on~\citep{bricken2023monosemanticity,templeton2024scaling,rajamanoharan2024improving,cunningham2023sparse,bills2023language}. The rating of a feature's interpretability is usually either done by human raters or by querying a language model. In the following two sections we evaluate the interpretability of JumpReLU, Gated and TopK SAE features using both a blinded human rating study, similar to~\citet{bricken2023monosemanticity,rajamanoharan2024improving}, and automated ratings using a language model, similar to~\citet{bricken2023monosemanticity,bills2023language,cunningham2023sparse,lieberum2024interpreting}.

\subsubsection{Manual Interpretability}
\label{sec:interpretability-study}
\paragraph{Methodology}
Our experimental setup closely follows~\citet{rajamanoharan2024improving}.
For each sublayer (Attention Output, MLP Output, Residual Stream), each layer (9, 20, 31) and each architecture (Gated, TopK, JumpReLU) we picked three SAEs to study, for a total of 81 SAEs. SAEs were selected based on their average number of active features. We selected those SAEs which had an average number of active features closest to 20, 75 and 150.

Each of our 5 human raters was presented with summary information and activating examples from the full activation spectrum of a feature. A rater rated a feature from every SAE, presented in a random order. The rater then decided whether a feature is mostly monosemantic based on the information provided, with possible answer options being `Yes', `Maybe', and `No', and supplied a short explanation of the feature where applicable. In total we collected 405 samples, i.e.~5 per SAE.

\paragraph{Results}

In \cref{fig:manual_interp_main}, we present the results of the manual interpretability study. Assuming a binomial 1-vs-all distribution for each ordinal rating value, we report the 2.5th to 97.5th percentile of this distribution as confidence intervals. All three SAE varieties exhibit similar rating distributions, consistent with prior results comparing TopK and Gated SAEs~\citep{anthropic_june_update,gao2024scalingevaluatingsparseautoencoders} and furthermore showing that JumpReLU SAEs are similarly interpretable.

\begin{figure}[t]
\centering
\includegraphics[width=1.0\columnwidth]{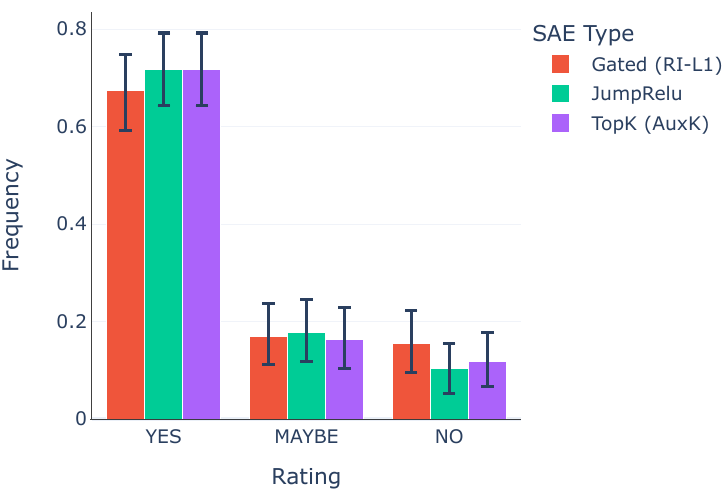}
\caption{Human rater scores of feature interpretability. Features from all SAE architectures are rated as similarly interpretable by human raters.}
\label{fig:manual_interp_main}
\end{figure}

\subsubsection{Automated Interpretability}
\label{sec:automated-interp}
In contrast to the manual rating of features, automated rating schemes have been proposed to speed up the evaluation process. The most prominent approach is a two step process of generating an explanation for a given feature with a language model and then predicting the feature's activations based on that explanation, again utilizing a language model. This was initially proposed by~\citet{bills2023language} for neurons, and later employed by~\citet{bricken2023monosemanticity,lieberum2024interpreting,cunningham2023sparse} for learned SAE features.

\paragraph{Methodology}
We used Gemini Flash \citep{geminiteam2024gemini15unlockingmultimodal} for explanation generation and activation simulation. In the first step, we presented Gemini Flash with a list of sequences that activate a given feature to different degrees, together with the activation values. The activation values were binned and normalized to be integers between 0 and 10. Gemini Flash then generated a natural language explanation of the feature consistent with the activation values.

In the second step we asked Gemini Flash to predict the activation value for each token of the sequences that were used to generate the explanations\footnote{Note that the true activation values were not known to the model at simulation time.}. We then computed the correlation between the simulated and ground truth activation values. We found that using a diverse few-shot prompt for both explanation generation and activation simulation was important for consistent results.

We computed the correlation score for 1000 features of each SAE, i.e.~three architectures, three layers, three layers/sub-layers and five or six sparsity levels, or 154 SAEs in total.

\paragraph{Results}

We show the distribution of Pearson correlations between language model simulated and ground truth activations in \cref{fig:auto_interp_main}. There is a small but notable improvement in mean correlation from Gated to JumpReLU and from JumpReLU. Note however, that the means clearly do not capture the extent of the within-group variation. We also report a baseline of explaining the activations of a randomly initialized JumpReLU SAE for the layer 20 residual stream -- effectively producing random, clipped projections of the residual stream. This exhibits markedly worse correlation scores, though notably with a clearly non-zero mean. We show the results broken down by site and layer in \cref{fig:auto_interp_app}. Note that in all of these results we are grouping together SAEs with very different sparsity levels and corresponding performances.

\begin{figure}[t]
\centering
\includegraphics[width=1.0\columnwidth]{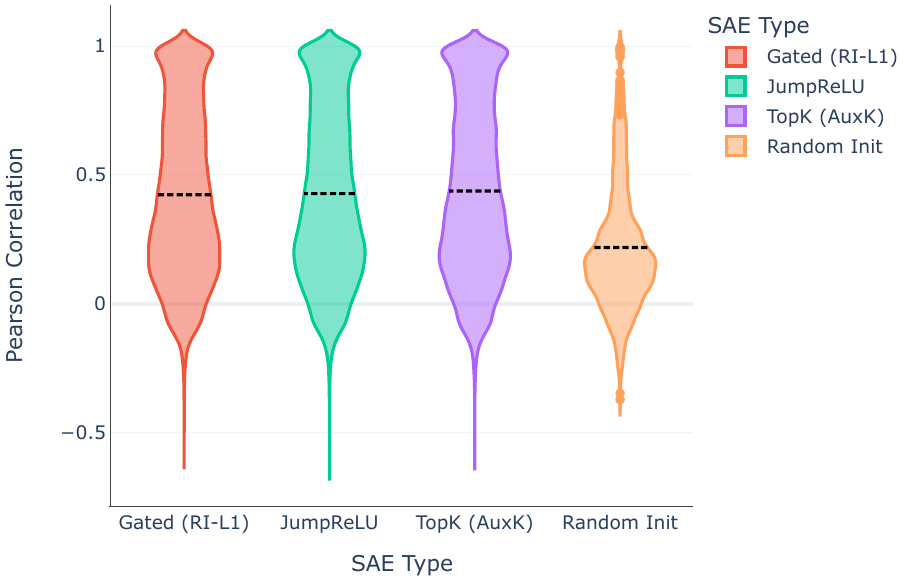}
\caption{Pearson correlation between LM-simulated and ground truth activations. The dashed lines denote the mean per SAE type. Values above 1 are an artifact of the kernel density estimation used to produce the plot.}
\label{fig:auto_interp_main}
\end{figure}

\section{Related work}

Recent interest in training SAEs on LM activations \citep{sharkey2022interim,bricken2023monosemanticity,cunningham2023sparse} stems from the twin observations that many concepts appear to be linearly represented in LM activations \citep{elhage2021mathematical,gurnee2023finding,olah2020zoom,park2023linear} and that dictionary learning \citep{mallat1993matchingpursuits,olshausen} may help uncover these representations at scale. It is also hoped that the sparse representations learned by SAEs may be a better basis for identifying computational subgraphs that carry out specific tasks in LMs \citep{wang2023interpretability,conmy2023automated,dunefsky2024transcodersinterpretablellmfeature} and for finer-grained control over LMs' outputs \citep{conmy2024activation,templeton2024scaling}.

Recent improvements to SAE architectures -- including TopK SAEs \citep{gao2024scalingevaluatingsparseautoencoders} and Gated SAEs \citep{rajamanoharan2024improving} -- as well as improvements to initialization and sparsity penalties. \cite{conerly2024trainingsaes} have helped ameliorate the trade-off between sparsity and fidelity and overcome the challenge of SAE features dying during training. Like JumpReLU SAEs, both Gated and TopK SAEs possess a thresholding mechanism that determines which features to include in a reconstruction; indeed, with weight sharing, Gated SAEs are mathematically equivalent to JumpReLU SAEs, although they are trained using a different loss function. JumpReLU SAEs are also closely related to ProLU SAEs \citep{taggart}, which use a (different) STE to train an activation threshold, but do not match the performance of Gated or TopK SAEs \citep{gao2024scalingevaluatingsparseautoencoders}.

The activation function defined in \cref{eq:jump-relu} was named JumpReLU in \citet{erichson2019jumprelu}, although it appears in earlier work, such as the TRec function in \citet{konda2015zerobiasautoencodersbenefitscoadapting}. Both TopK and JumpReLU activation functions are closely related to activation pruning techniques such as ASH \citep{djurisic2023extremelysimpleactivationshaping}.

The term \emph{straight through estimator} was introduced in \citet{bengio2013estimatingpropagatinggradientsstochastic}, although it is an old idea.\footnote{Even the Perceptron learning algorithm \citep{rosenblatt1958perceptron} can be understood as using a STE to train through a step function discontinuity.} STEs have found applications in areas such as training quantized networks (e.g.~\citet{hubara2016quantizedneuralnetworkstraining}) and circumventing defenses to adversarial examples \citep{athalye2018obfuscatedgradientsfalsesense}. Our interpretation of STEs in terms of gradients of the expected loss is related to \citet{yin2019understandingstraightthroughestimatortraining}, although they do not make the connection between STEs and KDE. \citet{louizos2017learning} also show how it is possible to train models using a L0 sparsity penalty -- on weights rather than activations in their case -- by introducing stochasticity in the weights and taking the gradient of the expected loss.

\section{Discussion}

Our evaluations show that JumpReLU SAEs produce reconstructions that consistently match or exceed the faithfulness of TopK SAEs, and exceed the faithfulness of Gated SAEs, at a given level of sparsity. They also show that the average JumpReLU SAE feature is similarly interpretable to the average Gated or TopK SAE feature, according to manual raters and automated evaluations. Although JumpReLU SAEs do have relatively more very high frequency features than Gated SAEs, they are similar to TopK SAEs in this respect.

In light of these observations, and taking into account the efficiency of training with the JumpReLU loss -- which requires no auxiliary terms and does not involve relatively expensive TopK operations -- we consider JumpReLU SAEs to be a mild improvement over prevailing SAE training methodologies.

Nevertheless, we note two key limitations with our study:
\begin{compactitem}
\item The evaluations presented in this paper concern training SAEs on several sites and layers of a single model, Gemma 2 9B. This does raise uncertainty over how well these results would transfer to other models -- particularly those with slightly different architectural or training details. In mitigation, although we have not presented the results in this paper, our preliminary experiments with JumpReLU on the Pythia suite of models \citep{biderman2023pythia} produced very similar results, both when comparing the sparsity-fidelity trade off between architectures and comparing interpretability. Nevertheless we would welcome attempts to replicate our results on other model families.
\item The science of principled evaluations of SAE performance is still in its infancy. Although we measured feature interpretability -- both assessed by human raters and by the ability of Gemini Flash to predict new activations given activating examples -- it is unclear how well these measures correlate to the attributes of SAEs that actually make them useful for downstream purposes. It would be valuable to evaluate these SAE varieties on a broader selection of metrics that more directly correspond to the value SAEs add by aiding or enabling downstream tasks, such as circuit analysis or model control.
\end{compactitem}

Finally, JumpReLU SAEs do suffer from a few limitations that we hope can be improved with further work:
\begin{compactitem}
\item Like TopK SAEs, JumpReLU SAEs tend to have relatively more very high frequency features -- features that are active on more than 10\% of tokens -- than Gated SAEs. Although it is hard to see how to reduce the prevalence of such features with TopK SAEs, we expect it to be possible to further tweak the loss function used to train JumpReLU SAEs to directly tackle this phenomenon.\footnote{Although, it could be the case that by doing this we end up pushing the fidelity-vs-sparsity curve for JumpReLU SAEs back closer to those of Gated SAEs. I.e.~it is plausible that Gated SAEs are close to the Pareto frontier attainable by SAEs that do not possess high frequency features.}
\item JumpReLU SAEs introduce new hyperparameters -- namely the initial value of $\thetabf$ and the bandwidth parameter $\varepsilon$ -- that require selecting. In practice, we find that, with dataset normalization in place, the default hyperparameters used in our experiments (\cref{app:further-training-details}) transfer quite reliably to other models, sites and layers. Nevertheless, there may be more principled ways to choose these hyperparameters, for example by adopting approaches to automatically selecting bandwidths from the literature on kernel density estimation.
\item The STE approach introduced in this paper is quite general. For example, we have also used STEs to train JumpReLU SAEs that have a sparsity level closed to some desired target $L_0^\text{target}$ by using the sparsity loss
\begin{equation}
    \mathcal{L}_\text{sparsity}(\mathbf{x}) = \lambda \left(\norm{\mathbf{f}(\mathbf{x})}_0 / L_0^\text{target} - 1\right)^2,
\end{equation}
much as it is possible to fix the sparsity of a TopK SAE by setting $K$ (see \cref{app:l0-target}). STEs thus open up the possibility of training SAEs with other discontinuous loss functions that may further improve SAE quality or usability.
\end{compactitem}

\section{Acknowledgements}

We thank Lewis Smith for reviewing the paper, including checking its mathematical derivations, and for valuable contributions to the SAE training codebase. We also thank Tom Conerly and Tom McGrath for pointing out errors in an earlier version of \cref{app:pseudo-code}. Finally, we are grateful to Rohin Shah and Anca Dragan for their sponsorship and support during this project.

\section{Author contributions}

Senthooran Rajamanoharan (SR) conceived the idea of training JumpReLU SAEs using the gradient of the expected loss, and developed the approach of using STEs to estimate this gradient. SR also performed the hyperparameter studies and trained the SAEs used in all the experiments. SAEs were trained using a codebase that was designed and implemented by Vikrant Varma and Tom Lieberum (TL) with significant contributions from Arthur Conmy, which in turn relies on an interpretability codebase written in large part by J\'anos Kram\'ar. TL was instrumental in scaling up the SAE training codebase so that we were able to iterate effectively on a 9B sized model for this project. TL also ran the SAE evaluations and manual interpretability study presented in the Evaluations section. Nicolas Sonnerat (NS) and TL designed and implemented the automated feature interpretation pipeline used to perform the automated interpretability study, with NS also leading the work to scale up the pipeline. SR led the writing of the paper, with the interpretability study sections and \cref{app:more-benchmarking-results} contributed by TL. Neel Nanda provided leadership and advice throughout the project and edited the paper.
\bibliography{main}


\appendix

\section{Differentiating integrals involving Heaviside step functions}

We start by reviewing some results about differentiating integrals (and expectations) involving Heaviside step functions.

\begin{lemma}
\label{lem:y-density}
Let $\mathbf{X}$ be a $n$-dimensional real random variable with probability density $p_\mathbf{X}$ and let $Y = g(\mathbf{X})$ for a differentiable function $g:\mathbb{R}^n \to \mathbb{R}$. Then we can express the probability density function of $Y$ as the surface integral
\begin{equation}
    p_Y(y) = \int_{\partial V(y)} p_\mathbf{X}(\mathbf{x}') \ud S
    \label{eq:y-density}
\end{equation}
where $\partial V(y)$ is the surface $g(\mathbf{x}) = y$ and $\ud S$ is its surface element.
\end{lemma}
\begin{proof}
From the definition of a probability density function:
\begin{align}
    p_Y(y) &:= \frac{\partial}{\partial y} \mathbb{P}\left(Y < y\right) \\
    &\phantom{:}= \frac{\partial}{\partial y} \int_{V(y)} p_\mathbf{X} (\mathbf{x}) \ud^n \mathrm{x}
\end{align}
where $V(y)$ is the volume $g(\mathbf{x}) < y$. \cref{eq:y-density} follows from an application of the multidimensional Leibniz integral rule.
\end{proof}

\begin{theorem}
\label{thm:ev-derivative}
Let $\mathbf{X}$ and $y$ once again be defined as in \cref{lem:y-density}. Also define
\begin{equation}
    A(y) := \mathbb{E}\left[f(\mathbf{X}) H(g(\mathbf{X}) - y))\right]
\end{equation}
where $H$ is the Heaviside step function for some function $f:\mathbb{R}^n \to \mathbb{R}$. Then, as long as $f$ is differentiable on the surface $g(\mathbf{x}) = y$, the derivative of $A$ at $y$ is given by
\begin{equation}
    A'(y) = - \mathbb{E}\left[f(\mathbf{X}) \vert Y=y \right] p_Y(y)
    \label{eq:cond-density}
\end{equation}
\end{theorem}
\begin{proof}
We can express $A(y)$ as the volume integral
\begin{equation}
    A(y) = \int_{V(y)} f(\mathbf{x}) p_\mathbf{X}(\mathbf{x}) \ud^n \mathbf{x}
\end{equation}
where $V(y)$ is now the volume $g(\mathbf{x}) > y$. Applying the multidimensional Leibniz integral rule (noting that $f$ is differentiable on the boundary of $V(y)$, we therefore obtain
\begin{equation}
    A'(y) = - \int_{\partial V(y)} f(\mathbf{x}) p_\mathbf{X}(\mathbf{x}) \ud S
    \label{eq:cond-density-as-integral}
\end{equation}
where $\partial V$ is the surface $g(\mathbf{x}) = y$. \cref{eq:cond-density} follows by noting that $p_\mathbf{X}(\mathbf{x}) = p_{\mathbf{X}|Y=y}(\mathbf{x}) p_Y(y)$ and thus substituting \cref{eq:y-density} into \cref{eq:cond-density-as-integral}.
\end{proof}

\begin{lemma}
With the same definitions as in \cref{thm:ev-derivative}, the expected value
\begin{equation}
    B(y) := \mathbb{E}\left[f(\mathbf{X}) H(g(\mathbf{X}) - y))^2\right],
\end{equation}
which involves the square of the Heaviside step function, is equal to $A(y)$.
\end{lemma}
\begin{proof}
Expressed in integral form, both $A(y)$ and $B(y)$ have the same domains of integration (the volume $g(\mathbf{x}) > y$) and integrands; therefore their values are identical.
\end{proof}

\section{Differentiating the expected loss}
\label{app:expected-loss-derivative}

The JumpReLU loss is given by
\begin{equation*}
    \mathcal{L}_\thetabf (\mathbf{x}) := \norm{\mathbf{x}-\hat{\mathbf{x}}(\mathbf{f}(\mathbf{x}))}_2^2 + \lambda \norm{\mathbf{f}(\mathbf{x})}_0.
    \tag{\ref{eq:jr-loss}}
\end{equation*}
By substituting in the following expressions for various terms in the loss:
\begin{align}
f_i(\mathbf{x}) &= \pi_i(\mathbf{x}) H(\pi_i(\mathbf{x}) - \theta_i), \\
\hat{x}(\mathbf{f}) &= \sum_{i=1}^M f_i(\mathbf{x}) \mathbf{d}_i + \bdec, \\
\norm{\mathbf{f}(\mathbf{x})}_0 &= \sum_{i=1}^M H(\pi_i(\mathbf{x}) - \theta_i),
\end{align}
taking the expected value, and differentiating (making use of the results of the previous section), we obtain
\begin{equation}
\frac{\partial \mathbb{E}_\mathbf{x}\left[\mathcal{L}_\thetabf(\mathbf{x})\right]}{\partial \theta_i} 
    = \left(\mathbb{E}_\mathbf{x} \left[J_i(\mathbf{x})\vert \pi_i(\mathbf{x}) = \theta_i\right] - \lambda\right) p_i(\theta_i)
    \label{eq:exact-derivative-j}
\end{equation}
where $p_i$ is the probability density function for the pre-activation $\pi_i(\mathbf{x})$ and
\begin{multline}
    J_i(\mathbf{x}) := 2 \theta_i \mathbf{d}_i \cdot \bigg[
    \mathbf{x} - \bdec - \tfrac{1}{2}\theta_i\mathbf{d}_i \\
    - \sum_{j\neq i}^{M} \pi_j(\mathbf{x})\mathbf{d}_j H(\pi_j(\mathbf{x}) - \theta_j)\bigg].
\end{multline}
We can express this derivative in the more succinct form given in \cref{eq:exact-derivative} and \cref{eq:i-definition} by defining
\begin{align}
    I_i(\mathbf{x}) &:= 2\theta_i \mathbf{d}_i \cdot \left[\mathbf{x} - \hat{\mathbf{x}}(\mathbf{f}(\mathbf{x}))\right] \\
    &\phantom{:}= 2\theta_i \mathbf{d}_i \cdot \bigg[\mathbf{x} - \bdec \\&\phantom{:= 2\theta_i \mathbf{d}_i \cdot \bigg[}-\sum_{j=1}^M \pi_j(\mathbf{x}) \mathbf{d}_j H(\pi_j(\mathbf{x}) - \theta_j)\bigg].
    \nonumber
\end{align}
and adopting the convention $H(0) := \tfrac{1}{2}$; this means that $I_i(\mathbf{x}) = J_i(\mathbf{x})$ whenever $\pi_i(\mathbf{x}) = \theta_i$, allowing us to replace $J_i$ by $I_i$ within the conditional expectation in \cref{eq:exact-derivative-j}.

\section{Using STEs to produce a kernel density estimator}
\label{app:ste-gives-kde}

Using the chain rule, we can differentiate the JumpReLU loss function to obtain the expression
\begin{multline}
    \frac{\partial \mathcal{L}_\thetabf(\mathbf{x})}{\partial \theta_i} =
    - \left(\frac{I_i(\mathbf{x})}{\theta_i}\right) \frac{\partial }{\partial \theta_i}\jr_{\theta_i}(\pi_i(\mathbf{x})) \\
    + \lambda \frac{\partial}{\partial \theta_i}H(\pi_i(\mathbf{x}) - \theta_i)
    \label{eq:chain-rule-on-loss}
\end{multline}
where $I_i(\mathbf{x})$ is defined as in \cref{eq:i-definition}. If we replace the partial derivatives in \cref{eq:chain-rule-on-loss} with the pseudo-derivatives defined in \cref{eq:jr-ste} and \cref{eq:step-ste}, we obtain the following expression for the pseudo-gradient of the loss:
\begin{equation}
    \frac{\pseudopartial \mathcal{L}_\thetabf(\mathbf{x})}{\pseudopartial \theta_i} = \frac{I_i(\mathbf{x}) - \lambda}{\varepsilon}K\left(\frac{\pi_i(\mathbf{x}) - \theta_i}{\varepsilon}\right).
\end{equation}
Computing this pseudo-gradient over a batch of observations $\mathbf{x}_1$, $\mathbf{x}_2$, $\ldots$, $\mathbf{x}_N$ and taking the mean, we obtain the kernel density estimator
\begin{equation*}
    \frac{1}{N\varepsilon}\sum_{\alpha=1}^{N} \left(I_i(\mathbf{x_\alpha}) - \lambda\right) K \left(\frac{
    \pi_i(\mathbf{x}_\alpha)-\theta_i}{\varepsilon}\right).
    \tag{\ref{eq:kde-estimate}}
\end{equation*}

\section{Combining Gated SAEs with the RI-L1 sparsity penalty}
\label{app:gated-sfn}
Gated SAEs compute two encoder pre-activations:
\begin{align}
    \pigate(\mathbf{x}) &:= \Wgate \mathbf{x} + \bgate, \\
    \pimag(\mathbf{x}) &:= \Wmag \mathbf{x} + \bmag.
\end{align}
The first of these is used to determine which features are active, via a Heaviside step activation function, whereas the second is used to determine active features' magnitudes, via a ReLU step function:
\begin{align}
    \fgate(\mathbf{x}) &:= H(\pigate(\mathbf{x})) \\
    \fmag(\mathbf{x}) &:= \text{ReLU}(\pimag(\mathbf{x})).
\end{align}
The encoder's overall output is given by the elementwise product $\mathbf{f}(\mathbf{x}) := \fgate(\mathbf{x}) \odot \fmag(\mathbf{x})$. The decoder of a Gated SAE takes the standard form
\begin{equation}
    \hat{\mathbf{x}}(\mathbf{f}) := \Wdec \mathbf{f} + \bdec.
    \tag{\ref{eq:decoder}}
\end{equation}
As in \citet{rajamanoharan2024improving}, we tie the weights of the two encoder matrices, parameterising $\Wmag$ in terms of $\Wgate$ and a vector-valued rescaling parameter $\rmag$:
\begin{equation}
    \left(\Wmag\right)_{ij} := \left(\exp(\rmag)\right)_i \left(\Wgate\right)_{ij}.
\end{equation}

The loss function used to train Gated SAEs in \citet{rajamanoharan2024improving} includes a L1 sparsity penalty and auxiliary loss term, both involving the positive elements of $\pigate$, as follows:
\begin{multline}
    \mathcal{L}_\text{gate} := \norm{\mathbf{x} - \hat{\mathbf{x}}(\mathbf{f}(\mathbf{x}))}_2^2
    + \lambda \norm{\text{ReLU}(\pigate(\mathbf{x}))}_1 \\ + \norm{\mathbf{x} - \hat{\mathbf{x}}_\text{frozen}(\text{ReLU}(\pigate(\mathbf{x})))}_2^2
    \label{eq:gated-loss}
\end{multline}
where $\hat{\mathbf{x}}_\text{frozen}$ is a frozen copy of the decoder, so that $\Wdec$ and $\bdec$ do not receive gradient updates from the auxiliary loss term.

For our JumpReLU evaluations in \cref{sec:experiments}, we also trained a variant of Gated SAEs where we replace the L1 sparsity penalty in \cref{eq:gated-loss} with the reparameterisation-invariant L1 (RI-L1) sparsity penalty $S_\rilo$ defined in \cref{eq:rilo-penalty}, i.e.~by making the replacement $\norm{\text{ReLU}(\pigate(\mathbf{x})}_1 \to S_\rilo(\pigate(\mathbf{x}))$, as well as unfreezing the decoder in the auxiliary loss term. As demonstrated in \cref{fig:paretos}, Gated SAEs trained this way have a similar sparsity-vs-fidelity trade-off to SAEs trained using the original Gated loss function, without the need to use resampling to avoid the appearance of dead features during training.

\section{Approximating TopK}
\label{app:top-k}

We used the approximate TopK approximation \texttt{jax.lax.approx\_max\_k} \citep{chern2022tpuknnknearestneighbor} to train the TopK SAEs used in the evaluations in \cref{sec:experiments}. Furthermore, we included the AuxK auxiliary loss function to train these SAEs. Supporting these decisions, \cref{fig:topk-variants} shows:
\begin{compactitem}
\item That SAEs trained with an approximate TopK activation function perform similarly to those trained with an exact TopK activation function;
\item That the AuxK loss slightly improves reconstruction fidelity at a given level of sparsity.
\end{compactitem}

\begin{figure}[t]
\centering
\includegraphics[width=\columnwidth]{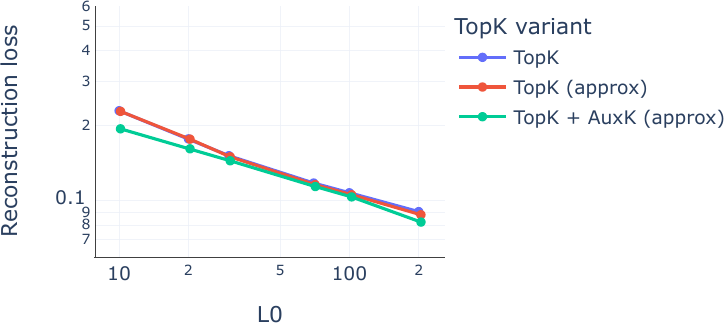}
\caption{Using an approximation of TopK leads to similar performance as exact TopK. Adding the AuxK term to the loss function slightly improves fidelity at a given level of sparsity.}
\label{fig:topk-variants}
\end{figure}

\section{Training JumpReLU SAEs to match a desired level of sparsity}
\label{app:l0-target}
Using the same pseudo-derivatives defined in \cref{sec:jumprelu-saes} it is possible to train JumpReLU SAEs with other loss functions. For example, it may be desirable to be able to target a specific level of sparsity during training -- as is possible by setting $K$ when training TopK SAEs -- instead of the sparsity of the trained SAE being an implicit function of the sparsity coefficient and reconstruction loss.

A simple way to achieve this is by training JumpReLU SAEs with the loss
\begin{equation}
    \mathcal{L}(\mathbf{x}) := \norm{\mathbf{x}-\hat{\mathbf{x}}(\mathbf{f}(\mathbf{x}))}_2^2 + \lambda \left(\frac{\norm{\mathbf{f}(\mathbf{x})}_0}{L_0^\text{target}} - 1\right)^2.
    \label{eq:target-l0-loss}
\end{equation}
Training SAEs with this loss on Gemma 2 9B's residual stream after layer 20, we find a similar fidelity-to-sparsity relationship to JumpReLU SAEs trained with the loss in \cref{eq:jr-loss}, as shown in \cref{fig:paretos-target}. Moreover, by using with the above loss, we are able to train SAEs that have L0s at convergence that are close to their targets, as shown by the proximity of the red dots in the figure to their respective vertical grey lines.

\begin{figure}[t]
\centering
\includegraphics[width=\columnwidth]{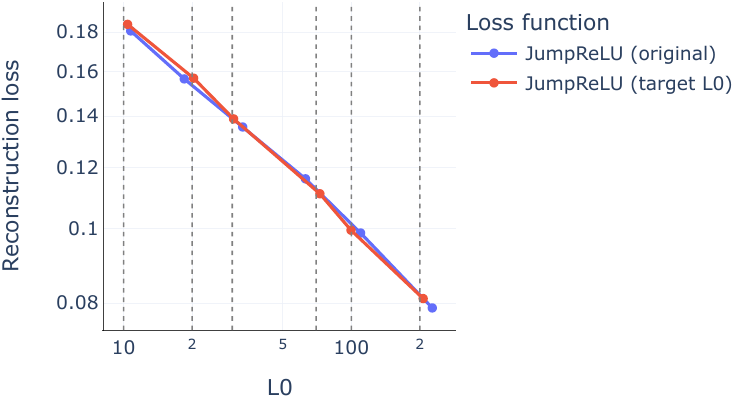}
\caption{By using the sparsity penalty in \cref{eq:target-l0-loss}, we can train JumpReLU SAEs to minimize reconstruction loss while maintaining a desired target level of sparsity. The vertical dashed grey lines indicate the target L0 values used to train the SAEs represented by the red dots closest to each line. These SAEs were trained setting $\lambda=1$.}
\label{fig:paretos-target}
\end{figure}

\section{Additional benchmarking results}
\label{app:more-benchmarking-results}

\cref{fig:paretos-mlp} and \cref{fig:paretos-attn} plot reconstruction fidelity against sparsity for SAEs trained on Gemma 2 9B MLP and attention outputs at layers 9, 20 and 31. \cref{fig:paretos-fvu} uses fraction of variance explained (see \cref{sec:experiments}) as an alternative measure of reconstruction fidelity, and again compares the fidelity-vs-sparsity trade-off for JumpReLU, Gated and TopK SAEs on MLP, attention and residual stream layer outputs for Gemma 2 9B layers 9, 20 and 31. \cref{fig:feature_frequency_histograms} compares feature activation frequency histograms for JumpReLU, TopK and Gated SAEs of comparable sparsity.

\begin{figure*}[p]
\centering
\includegraphics[width=\textwidth]{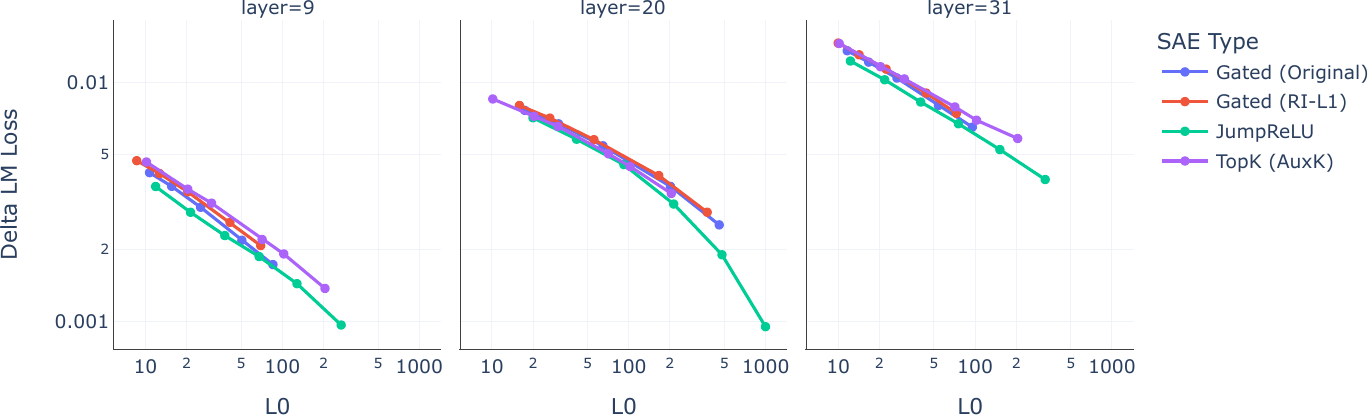}
\caption{Comparing reconstruction fidelity versus sparsity for JumpReLU, Gated and TopK SAEs trained on Gemma 2 9B layer 9, 20 and 31 MLP outputs. JumpReLU SAEs consistently provide more faithful reconstructions (lower delta LM loss) at a given level of sparsity (as measured by L0).}
\label{fig:paretos-mlp}
\end{figure*}

\begin{figure*}[p]
\centering
\includegraphics[width=\textwidth]{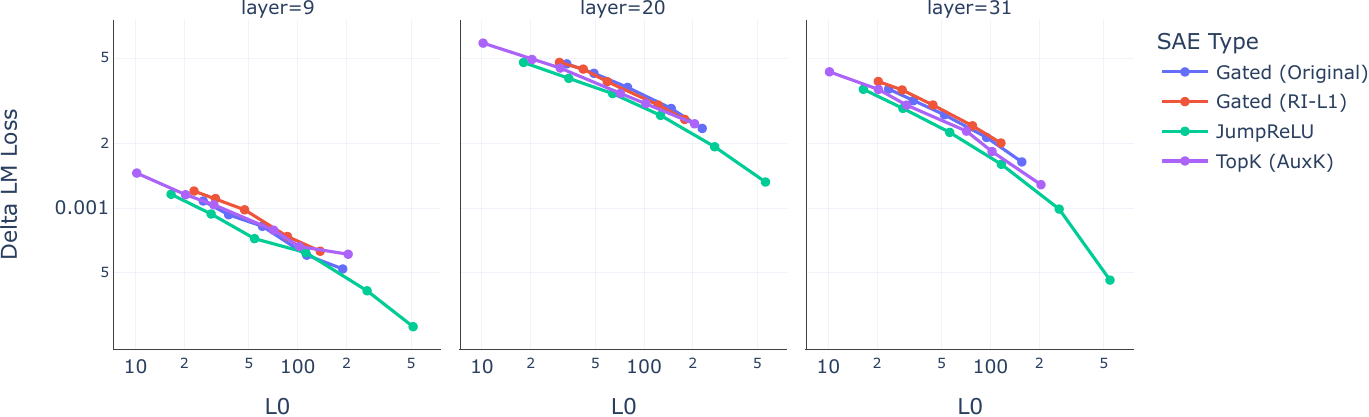}
\caption{Comparing reconstruction fidelity versus sparsity for JumpReLU, Gated and TopK SAEs trained on Gemma 2 9B layer 9, 20 and 31 attention activations prior to the attention output linearity ($\mathbf{W}_O$). JumpReLU SAEs consistently provide more faithful reconstructions (lower delta LM loss) at a given level of sparsity (as measured by L0).}
\label{fig:paretos-attn}
\end{figure*}

\begin{figure*}[p]
\centering
\includegraphics[width=\textwidth]{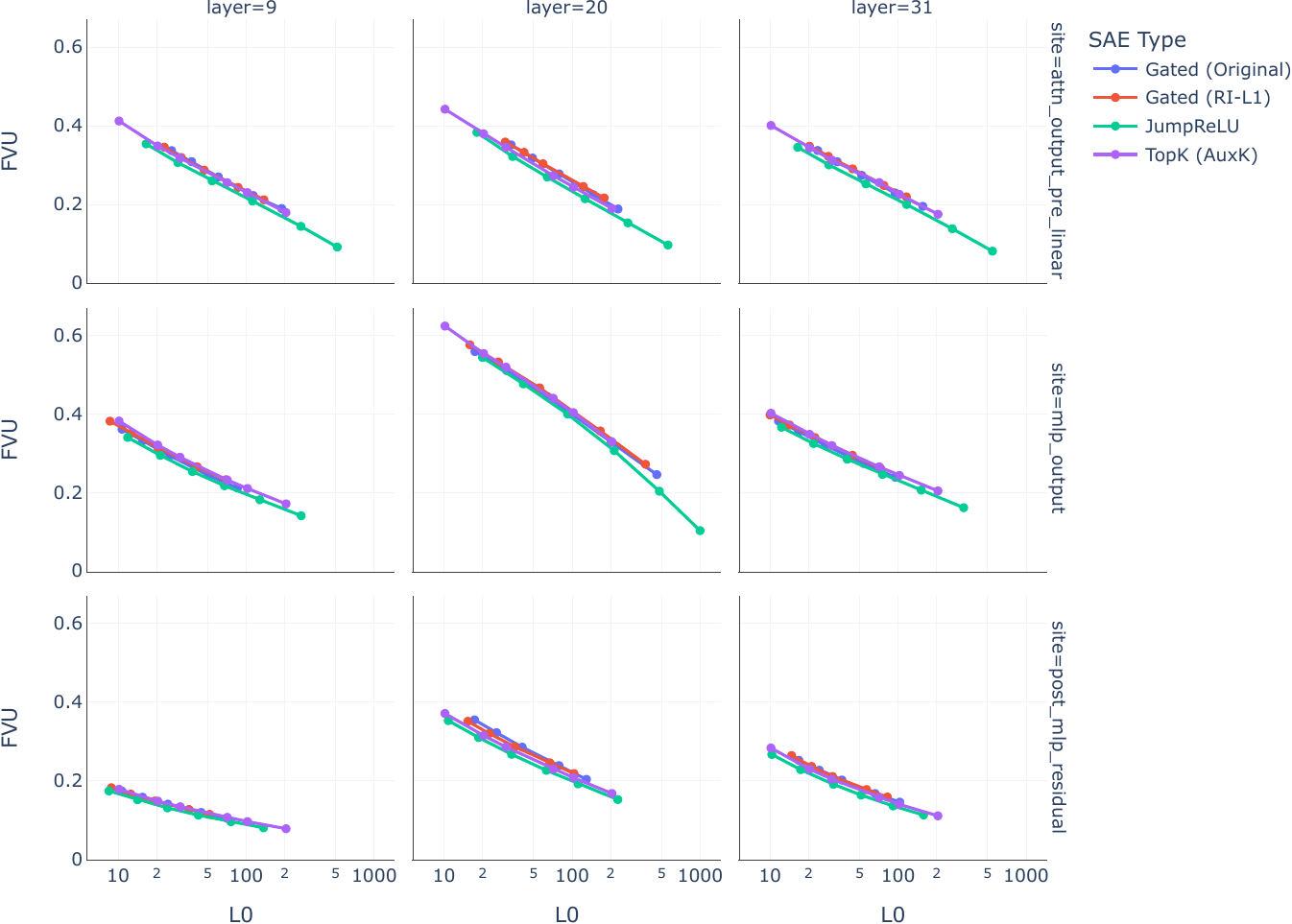}
\caption{Comparing reconstruction fidelity versus sparsity for JumpReLU, Gated and TopK SAEs trained on Gemma 2 9B layer 9, 20 and 31 MLP, attention and residual stream activations using fraction of variance unexplained (FVU) as a measure of reconstruction fidelity.}
\label{fig:paretos-fvu}
\end{figure*}

\begin{figure*}[p]
\centering
\includegraphics[width=\textwidth]{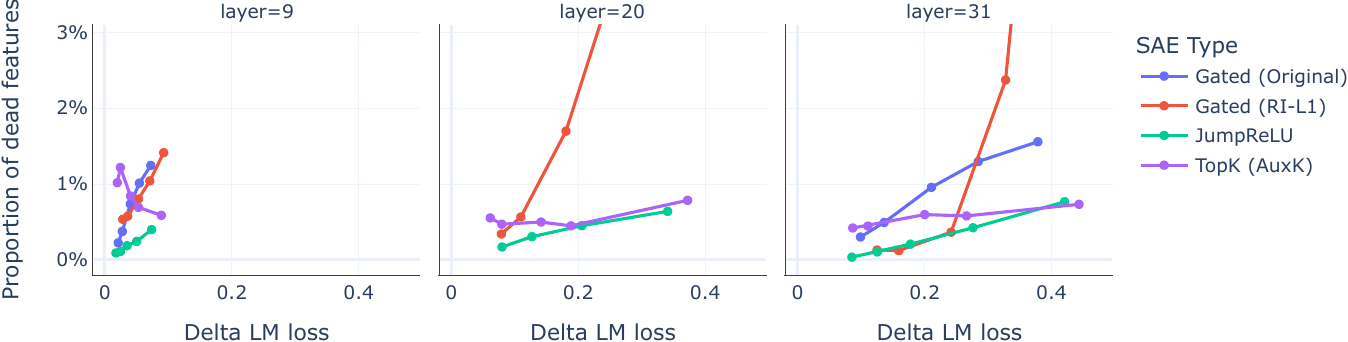}
\caption{JumpReLU and TopK SAEs have few dead features (features that activate on fewer than one in $10^7$ tokens), even without resampling. Note that the original Gated loss (blue) -- the only training method that uses resampling -- had around 40\% dead features at layer 20 and is therefore missing from the middle plot.}
\label{fig:dead-features}
\end{figure*}

\begin{figure*}[p]
\centering
\includegraphics[width=\textwidth]{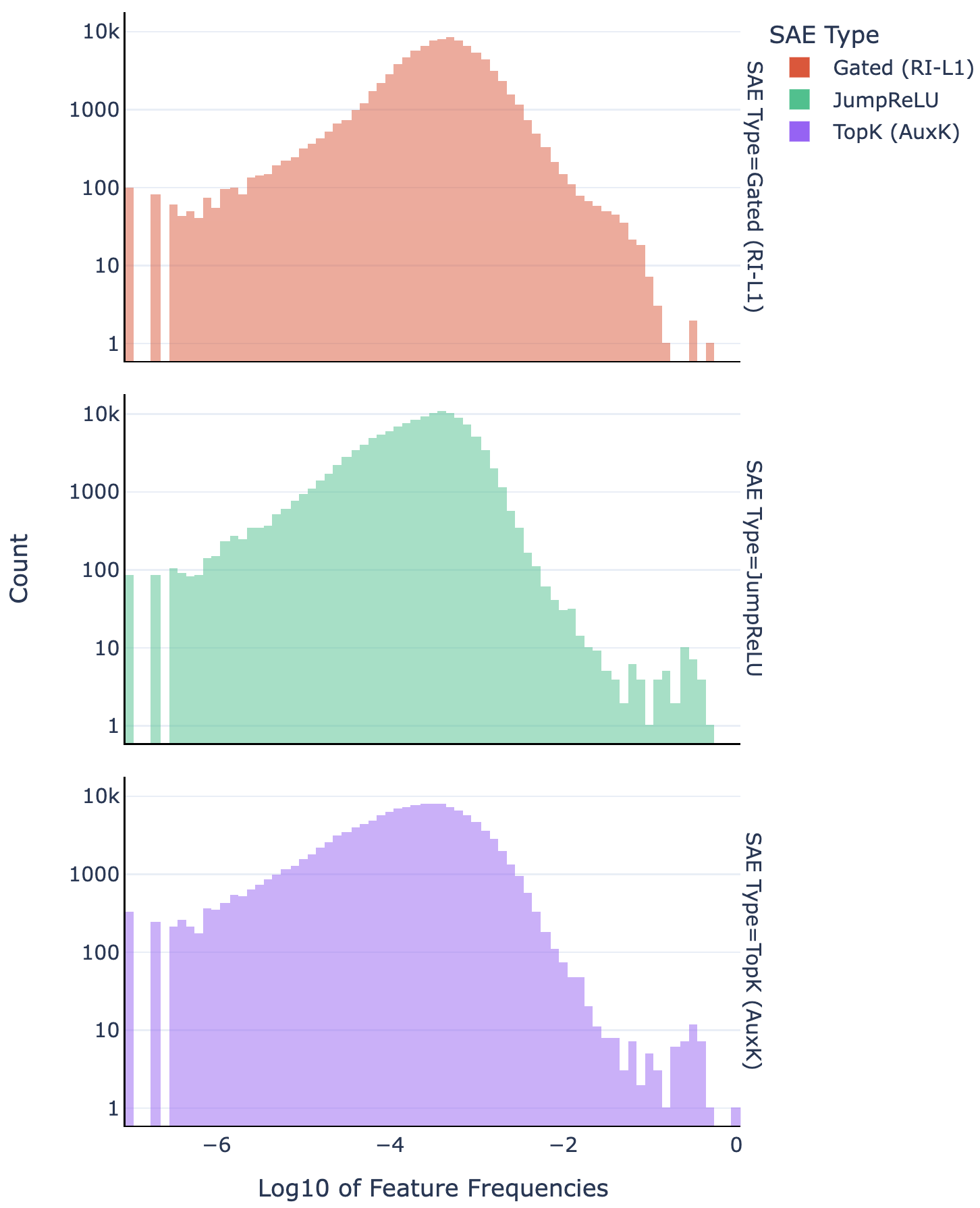}
\caption{Feature frequency histograms for JumpReLU, TopK and Gated SAEs all with L0 approximately 70 (excluding features with zero activation counts). Note the log-scale on the y-axis: this is to highlight a small mode of high frequency features present in the JumpReLU and TopK SAEs. Gated SAEs do not have this mode, but do have a ``shoulder'' of features with frequencies between $10^{-2}$ and $10^{-1}$ not present in the JumpReLU and TopK SAEs.}
\label{fig:feature_frequency_histograms}
\end{figure*}

\paragraph{Automated interpretability}
In fig \cref{fig:auto_interp_app} we show the distribution and means of the correlations between LM-simulated and ground truth activations, broken down by layer and site. In line with our other findings, layer 20 and the pre-linear attention output seem to perform worst on this metric.

\begin{figure*}[p]
\centering
\includegraphics[width=0.9\textwidth]{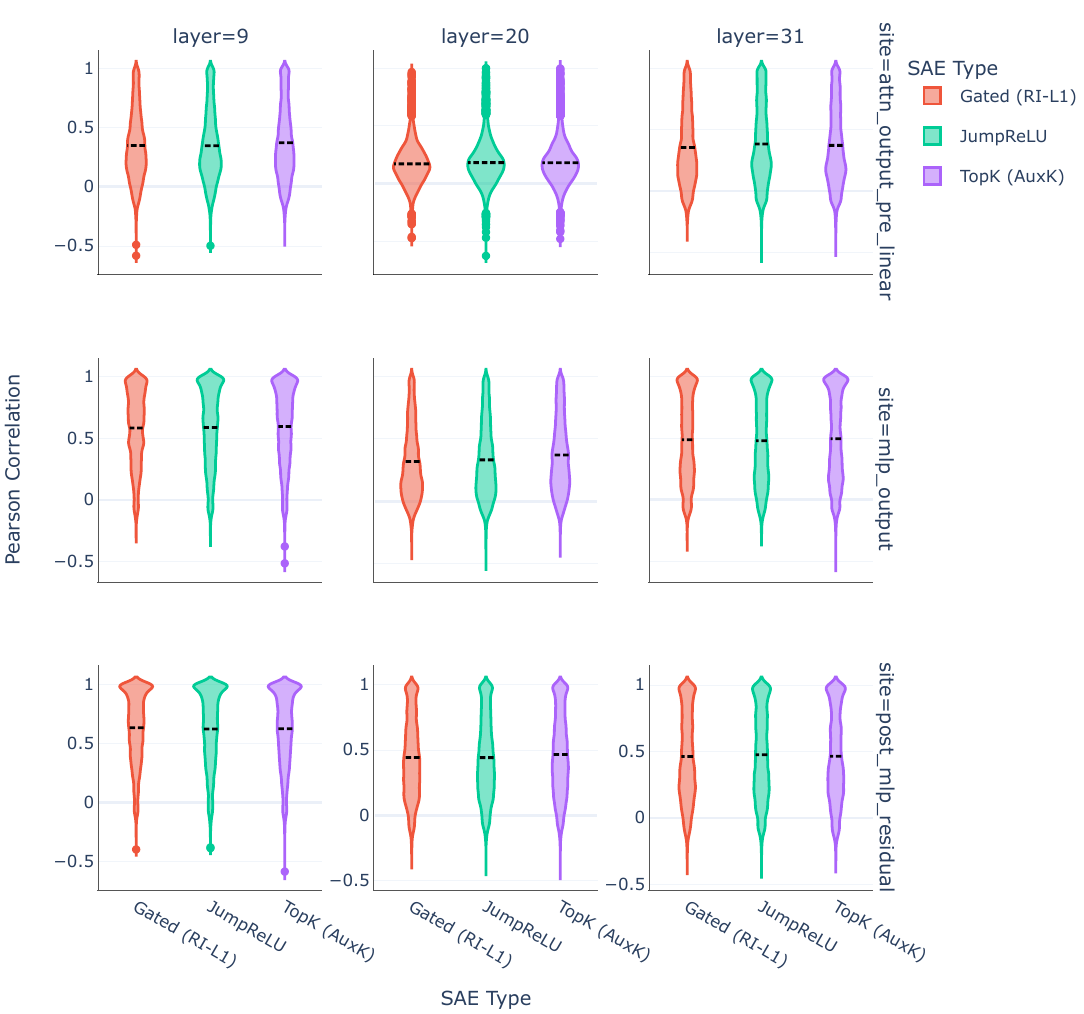}
\caption{Pearson correlation between simulated and ground truth activations, broken down by site and layer.}
\label{fig:auto_interp_app}
\end{figure*}

\begin{figure*}[p]
\centering
\includegraphics[width=\textwidth]{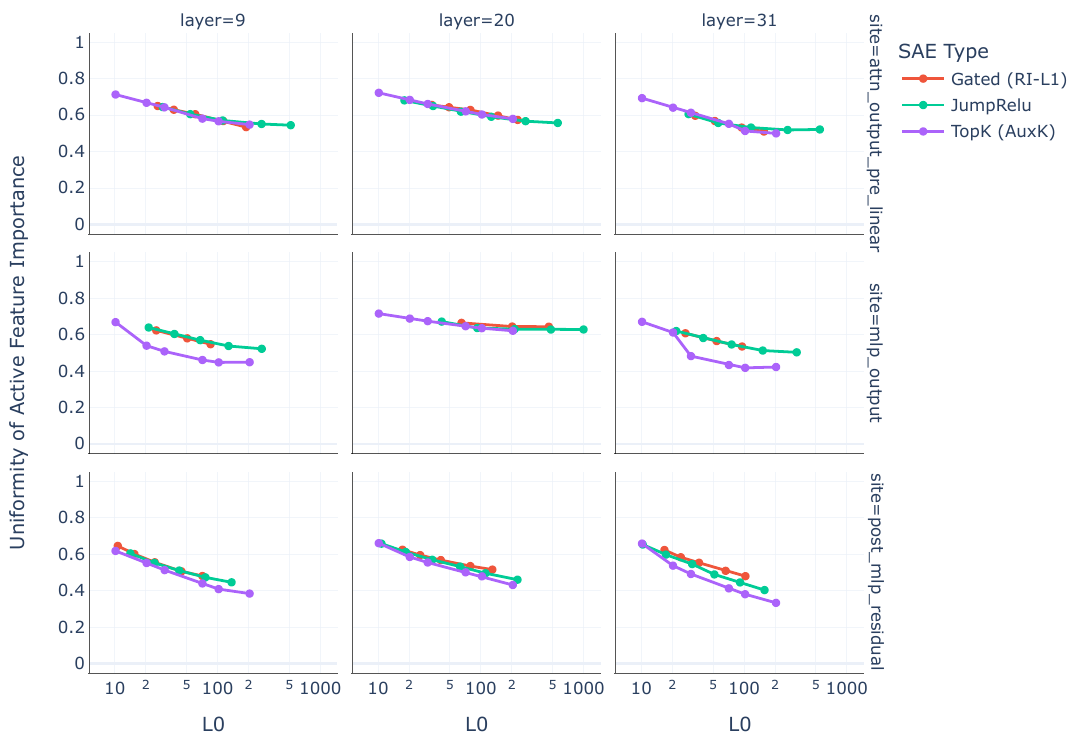}
\caption{Comparing uniformity of active feature importance against L0 for JumpReLU, Gated and TopK SAEs. All SAEs diffuse their effects more with increased L0. This effect appears strongest for TopK SAEs.}
\label{fig:effective-l0-ratio}
\end{figure*}

\paragraph{Attribution Weighted Effective Sparsity}
Conventionally, sparsity of SAE feature activations is measured as the L0 norm of the feature activations. \Citet{olah2024open} suggest to train SAEs to have low L1 activation of attribution-weighted feature activations, taking into account that some features may be more important than others. Inspired by this, we investigate the sparsity of the attribution weighted feature activations. Following \citet{olah2024open}, we define the attribution-weighted feature activation vector $\mathbf{y}$ as 

\begin{align*}
    \mathbf{y} := \mathbf{f}(\mathbf{x}) \odot \Wdec^T \nabla_\mathbf{x} \mathcal{L},
\end{align*}

where we choose the mean-centered logit of the correct next token as the loss function $\mathcal{L}$.
We then normalize the magnitudes of the entries of $\mathbf{y}$ to obtain a probability distribution $p \equiv p(\mathbf{y})$. We can measure how far this distribution diverges from a uniform distribution $u$ over active features via the KL divergence
\begin{align*}
    \mathbf{D}_{\text{KL}}(p \| u) = \log \|\mathbf{y}\|_0 - \mathbf{S}(p),
\end{align*}
with the entropy $\mathbf{S}(p)$. Note that $0 \leq \mathbf{D}_{\text{KL}}(p \| u)\leq \log \|\mathbf{y}\|_0$. Exponentiating the negative KL divergence gives a new measure $r_{L0}$ 
\begin{align*}
    r_{L0} := e^{-\mathbf{D}_{\text{KL}}(p \| u)} = \frac{e^{\mathbf{S}(p)}}{\|\mathbf{y}\|_0},
\end{align*}
with $\frac{1}{\|\mathbf{y}\|_0} \leq r_{L0} \leq 1$. Note that since $e^\mathbf{S}$ can be interpreted as the effective number of active elements, $r_{L0}$ is the ratio of the effective number of active features (after reweighting) to the total number of active features, which we call the `Uniformity of Active Feature Importance'.
We computed $r_{L0}$ over 2048 sequences of length 1024 (ignoring special tokens) for all SAE types and sparsity levels and report the result in \cref{fig:effective-l0-ratio}.
For all SAE types and locations, the more features are active the more diffuse their effect appears to be. Furthermore, this effect seems to be strongest for TopK SAEs, while Gated and JumpReLU SAEs behave mostly identical (except for layer 31, residual stream SAEs). However, we caution to not draw premature conclusions about feature quality from this observation.

\section{Using other kernel functions}
\label{app:other-kernels}

As described in \cref{sec:jumprelu-saes}, we used a simple rectangle function as the kernel, $K(z)$, within the pseudo-derivatives defined in \cref{eq:jr-ste} and \cref{eq:step-ste}. As shown in \cref{fig:other-kernels}, similar results can be obtained with other common KDE kernel functions; there does not seem to be any obvious benefit to using a higher order kernel.

\begin{figure}[t]
\centering
\includegraphics[width=\columnwidth]{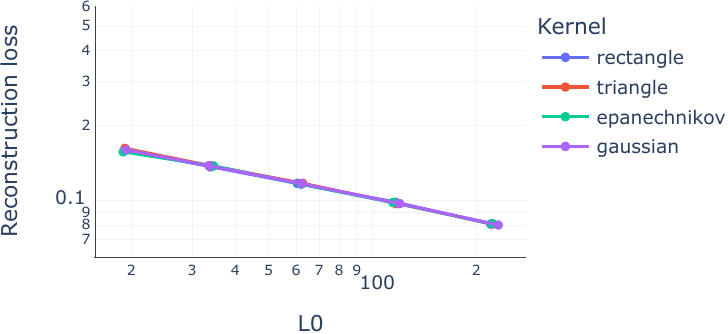}
\caption{Using different kernel functions to compute the pseudo-derivatives defined in \cref{eq:jr-ste} and \cref{eq:step-ste} has little impact on fidelity-vs-sparsity curves. These curves are for Gemma 2 9B post-layer 20 residual stream SAEs trained on 2B tokens.}
\label{fig:other-kernels}
\end{figure}

\section{Further details on our training methodology}
\label{app:further-training-details}

\begin{compactitem}
\item We normalise LM activations so that they have mean squared L2 norm of one during SAE training. This helps to transfer hyperparameters between different models, sites and layers.
\item We trained all our SAEs with a learning rate of $7\times10^{-5}$ and batch size of 4,096.
\item As in \citet{rajamanoharan2024improving}, we warm up the learning rate over the first 1,000 steps (4M tokens) using a cosine schedule, starting the learning rate at 10\% of its final value (i.e.~starting at $7\times10^{-6}$).
\item We used the Adam optimizer \citep{kingma2017adammethodstochasticoptimization} $\beta_1 = 0$, $\beta_2=0.999$ and $\epsilon=10^{-8}$. In our initial hyperparameter study, we found training with lower momentum ($\beta_1 < 0.9$) produced slightly better fidelity-vs-sparsity carves for JumpReLU SAEs, although differences were slight.
\item We use a pre-encoder bias during training \cite{bricken2023monosemanticity} -- i.e.~subtract $\bdec$ from $\mathbf{x}$ prior to the encoder. Through ablations we found this to either have no impact or provide a small improvement to performance (depending on model, site and layer).
\item For JumpReLU SAEs we initialised the threshold $\thetabf$ to 0.001 and the bandwidth $\varepsilon$ also to 0.001. These parameters seem to work well for a variety of LM sizes, from single layer models up to and including Gemma 2 9B.
\item For Gated RI-L1 SAEs we initialised the norms of the decoder columns $\norm{\mathbf{d}_i}_2$ to 0.1.
\item We trained all SAEs except for Gated RI-L1 while constraining the decoder columns $\norm{\mathbf{d}_i}_2$ to 1.\footnote{This is not strictly necessary for JumpReLU SAEs and we subsequently found that training JumpReLU SAE without this constraint does not change fidelity-vs-sparsity curves, but we have not fully explored the consequences of turning this constraint off.}
\item Following \citet{conerly2024trainingsaes} we set $\Wenc$ to be the transpose of $\Wdec$ at initialisation (but thereafter left the two matrices untied) when training of all SAE types, and warmed up $\lambda$ linearly over the first 10,000 steps (40M tokens) for all except TopK SAEs.
\item We used resampling \citep{bricken2023monosemanticity} -- periodically re-initialising the parameters corresponding to dead features -- with Gated (original loss) SAEs, but did not use resampling with Gated RI-L1, TopK or JumpReLU SAEs.
\end{compactitem}

\section{Pseudo-code for implementing and training JumpReLU SAEs}

We include pseudo-code for implementing:
\begin{compactitem}
\item The Heaviside step function with custom backward pass defined in \cref{eq:step-ste}.
\item The JumpReLU activation function with custom backward pass defined in \cref{eq:jr-ste}.
\item The JumpReLU SAE forward pass.
\item The JumpReLU loss function.
\end{compactitem}

Our pseudo-code most closely resembles how these functions can be implemented in JAX, but should be portable to other frameworks, like PyTorch, with minimal changes.

Two implementation details to note are:
\begin{compactitem}
\item We use the logarithm of threshold, i.e.~$\log(\thetabf)$, as our trainable parameter, to ensure that the threshold remains positive during training.
\item Even with this parameterisation, it is possible for the threshold to become smaller than half the bandwidth, i.e.~that $\theta_i < \varepsilon/2$ for some $i$. To ensure that negative pre-activations can never influence the gradient computation, we take the ReLU of the pre-activations before passing these to the JumpReLU activation function or the Heaviside step function used to compute the L0 sparsity term. Mathematically, this has no impact on the forward pass (because pre-activations below the positive threshold are set to zero in both cases anyway), but it ensures that negative pre-activations cannot bias gradient estimates in the backward pass. 
\end{compactitem}

\onecolumn

\begin{lstlisting}[language=Python]
def rectangle(x):
  return ((x > -0.5) & (x < 0.5)).astype(x.dtype)


### Implementation of step function with custom backward

@custom_vjp
def step(x, threshold):
  return (x > threshold).astype(x.dtype)


def step_fwd(x, threshold):
  out = step(x, threshold)
  cache = x, threshold  # Saved for use in the backward pass
  return out, cache


def step_bwd(cache, output_grad):
  x, threshold = cache
  x_grad = 0.0 * output_grad  # We don't apply STE to x input
  threshold_grad = (
      -(1.0 / bandwidth) * rectangle((x - threshold) / bandwidth) * output_grad
  )
  return x_grad, threshold_grad


step.defvjp(step_fwd, step_bwd)


### Implementation of JumpReLU with custom backward for threshold

@custom_vjp
def jumprelu(x, threshold):
  return x * (x > threshold)


def jumprelu_fwd(x, threshold):
  out = jumprelu(x, threshold)
  cache = x, threshold  # Saved for use in the backward pass
  return out, cache


def jumprelu_bwd(cache, output_grad):
  x, threshold = cache
  x_grad = (x > threshold) * output_grad  # We don't apply STE to x input
  threshold_grad = (
      -(threshold / bandwidth)
      * rectangle((x - threshold) / bandwidth)
      * output_grad
  )
  return x_grad, threshold_grad


jumprelu.defvjp(jumprelu_fwd, jumprelu_bwd)


### Implementation of JumpReLU SAE forward pass and loss functions

def sae(params, x, use_pre_enc_bias):
  # Optionally, apply pre-encoder bias
  if use_pre_enc_bias:
    x = x - params.b_dec

  # Encoder - see accompanying text for why we take the ReLU
  # of pre_activations even though it isn't mathematically
  # necessary
  pre_activations = relu(x @ params.W_enc + params.b_enc)
  threshold = exp(params.log_threshold)
  feature_magnitudes = jumprelu(pre_activations, threshold)

  # Decoder
  x_reconstructed = feature_magnitudes @ params.W_dec + params.b_dec

  # Also return pre_activations, needed to compute sparsity loss
  return x_reconstructed, feature_magnitudes


### Implementation of JumpReLU loss

def loss(params, x, sparsity_coefficient, use_pre_enc_bias):
  x_reconstructed, feature_magnitudes = sae(params, x, use_pre_enc_bias)

  # Compute per-example reconstruction loss
  reconstruction_error = x - x_reconstructed
  reconstruction_loss = sum(reconstruction_error**2, axis=-1)

  # Compute per-example sparsity loss
  threshold = exp(params.log_threshold)
  l0 = sum(step(feature_magnitudes, threshold), axis=-1)
  sparsity_loss = sparsity_coefficient * l0

  # Return the batch-wise mean total loss
  return mean(reconstruction_loss + sparsity_loss, axis=0)
\end{lstlisting}

\label{app:pseudo-code}

\end{document}